\def\year{2020}\relax
\documentclass[letterpaper]{article} 
\usepackage{aaai20}  
\usepackage{times}  
\usepackage{helvet} 
\usepackage{courier}  
\usepackage[hyphens]{url}  
\usepackage{graphicx} 
\usepackage{graphicx}  
\nocopyright
\usepackage{soul}
\usepackage[hidelinks]{hyperref}
\usepackage[utf8]{inputenc}
\usepackage[small]{caption}
\usepackage{amsmath}
\usepackage{amsthm}
\usepackage{booktabs}
\usepackage{tikz}
\usepackage{subcaption}
\usepackage{multirow}
\usepackage{comment}
\bgroup

\newtheorem{thm}{Theorem}
\newtheorem{example}{Example}
\newtheorem{definition}{Definition}

\newcommand{\strips}{\textsc{Strips}}

\usepackage{algorithm} 
\usepackage{algorithmic} 

\title{{\strips} Action Discovery}

\author{Alejandro Su\'arez-Hern\'andez\textsuperscript{\rm 1}, Javier Segovia-Aguas\textsuperscript{\rm 1}, Carme Torras\textsuperscript{\rm 1}, Guillem Aleny\`a\textsuperscript{\rm 1} \\
\textsuperscript{\rm 1}Institut de Robòtica i Informàtica Industrial, CSIC-UPC.\\ Llorens i Artigas 4-6, 08028 Barcelona, Spain.\\ {\tt \{asuarez, jsegovia, torras, galenya\}@iri.upc.edu}
}

\begin{document}

\maketitle

\begin{abstract}
The problem of specifying high-level knowledge bases for planning becomes a hard task in realistic environments. This knowledge is usually handcrafted and is hard to keep updated, even for system experts. Recent approaches have shown the success of classical planning at synthesizing action models even when all intermediate states are missing. These approaches can synthesize action schemas in Planning Domain Definition Language (PDDL) from a set of execution traces each consisting, at least, of an initial and final state. In this paper, we propose a new algorithm to unsupervisedly synthesize ${\strips}$ action models with a classical planner when action signatures are unknown. In addition, we contribute with a compilation to classical planning that mitigates the problem of learning static predicates in the action model preconditions, exploits the capabilities of SAT planners with parallel encodings to compute action schemas and validate all instances. Our system is flexible in that it supports the inclusion of partial input information that may speed up the search. We show through several experiments how learnt action models generalize over unseen planning instances.

\end{abstract}

\section{Introduction}


Most work in Artificial Intelligence planning assumes models are given as input for {\em plan synthesis}~\cite{ghallab2004automated}. Nevertheless, the acquisition of such models in complex environments becomes a hard task even for system experts limiting the potential applications of planning \cite{kambhampati2007model}. Research in web-service composition \cite{carman2003web} and work-flow management \cite{blythe2004automatically} as planning is an example of bottleneck for generating the input specification. 

Kambhampati (2007) \cite{kambhampati2007model} proposes in his work a shallow or approximate planning domain representation such that it can be criticized and/or refined with experiences. In that sense, the problem of learning action models consists on inferring, synthesizing or refining action schemas that can be used for planning on new tasks in the same environment. However, frameworks for learning action models varies from their inputs, action granularity, techniques, and output structures \cite{arora2018review}. On the one hand, most frameworks require input traces with the action signatures (i.e. action names and parameters), objects and predicates such as ARMS~\cite{yang2007learning}, SLAF~\cite{amir2008learning}, and $\strips$-based compilations~\cite{aineto2018learning}. On the other hand, the LOCM family of algorithms~\cite{cresswell2009acquisition,cresswell2011generalised,gregory2015domain}, generate object-centric planning models so they only need sequences of grounded actions. The action learner methods \cite{amado2018goal} can produce (without guarantees) a set of PDDL actions given a set of image transitions. The aim of these frameworks is to learn action models like the one in Figure~\ref{sfig:visitall-schema}, that represents the action to move the robot between adjacent locations. Figure~\ref{sfig:visitall-grid} shows an instance of {\em visitall}, where the agent starts at the bottom-left corner and must visit all cells in a grid. Learning action models consists in automatically generating the domain that rules any problem in a given environment, e.g. the ${\sf move}$ action with its preconditions and effects in the {\em visitall} domain.

\begin{figure}[tbp]
\begin{subfigure}[b]{0.2\textwidth}
\centering
\includegraphics[width=\textwidth]{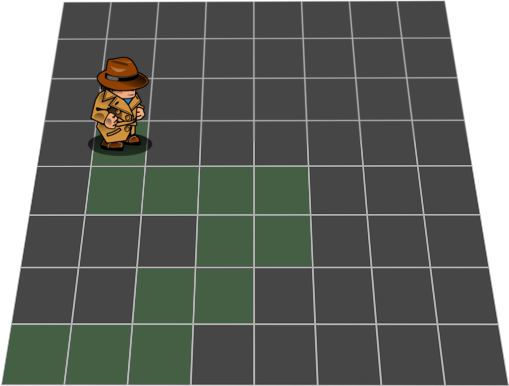}
\caption{Grid of $8 \times 8$ size.}
\label{sfig:visitall-grid}
\end{subfigure}
\begin{subfigure}[b]{0.3\textwidth}
\centering
\footnotesize
\begin{tabular}{l}
{\bf action :=} {\sf move( ?x ?y - place )}\\
{\bf pre :=} {\sf (agent-at ?x - place)$\wedge$}\\[-1.5mm]{\sf (connected ?x ?y - place)}\\
{\bf eff :=} {\sf (visited ?y - place)$\wedge$}\\[-1.5mm]{\sf ($\neg$ (agent-at ?x - place))}$\wedge$\\[-1.5mm]{\sf (agent-at ?y - place)}
\end{tabular}
\caption{Action schema: move }
\label{sfig:visitall-schema}
\end{subfigure}
\caption{{\em visit-all} example with the action schema.}
\vspace{-1em}
\label{fig:visitall-example}
\end{figure}

Motivated by recent advances on unsupervised learning approaches to apply planning in the latent space \cite{asai2018classical,amado2018goal}, unsupervised generalized model synthesis with classical planners \cite{segovia2017unsupervised}, 
and learning action models from minimal observability \cite{aineto2019learning}, we present in this paper a novel compilation to classical planning for learning $\strips$ action schemas like the one in Figure~\ref{sfig:visitall-schema} with an unsupervised algorithm. The input to the compilation is a set of goal-oriented examples, i.e. a set of planning instances with their set of objects, predicates, initial states and goal conditions. Then, action schemas are induced from the plan that solves the compiled planning problem. These are validated with the input examples in parallel, and evaluated on new instances to test their generalization. Moreover, these approaches can work with approximate models where a partial action schema is included as prior domain knowledge, but new actions can be learnt, validated on new evidences, and refine the domain model. The main contributions of this paper are:
\begin{itemize}
    \item a $\strips$-based compilation to learn {\em action signatures} from minimal information with a SAT-based planning system,
    \item a method to learn static predicates by removing preconditions of action models instead of being programmed,
    \item new actions in the compilations adapted to validate input examples in parallel with the learnt action models,
    \item an algorithm to explore learning tasks in the bounded space of configurations, and
    \item an empirical evaluation of action schemas convergence in regard to the increasing number of input examples.
\end{itemize} 

In the next section we introduce classical planning with conditional effects, action schemas in $\strips$, and formalize the problem of learning these actions with classical planning. Then, we formalize our task representation and its compilation to classical planning, followed by the configuration algorithm that generates the compilations. In the Experiments Section, we use the configuration algorithm to run experiments on learning and validating action models. Finally, we review related work and conclude the paper with some future work direction.

\section{Preliminaries}

In this section we introduce some preliminar notation of classical planning with conditional effects, and action schemas when defined in $\strips$ language~\cite{fikes1971strips} and how the problem of learning action schemas have been usually addressed.

\subsection{Classical planning with conditional effects}
A {\em classical planning problem} is defined by a 4-tuple $P = \langle F,A,I,G \rangle$, where $F$ is the set of fluents, $A$ is the set of actions, $I \subseteq F$ is the initial state and $G \subseteq F$ a goal condition. 

One extension of planning actions $A$, is to describe their effects conditionally affected by the current state $s$. Therefore, we formally define the state space $S = 2^{|L|}$ in terms of literals $L$, where a literal $l \in L$ is a valuation of a fluent $f \in F$, i.e. $l = f$ or $l = \neg f$. We assume $|s| = |L|$, $L$ does not assign conflicting values and $\neg L = \{\neg l : l \in L\}$ to be the complement of $L$.

Then, an action $a \in A$ with conditional effects is defined as $a = \langle \mathsf{pre}(a), \mathsf{ce}(a) \rangle$, where $\mathsf{pre}(a)$ are the preconditions and $\mathsf{ce}(a)$ is the set of conditional effects. We say that an action $a \in A$ is {\em applicable} in state $s \in S$ iff the preconditions hold in that state, i.e. $\mathsf{pre}(a) \subseteq s$. 
Moreover, we denote $C \triangleright E \in \mathsf{ce}(a)$ to the sets of literals that correspond to the condition $C$ and the effects $E$, and the whole set of {\em triggered effects} is defined as $\mathsf{eff}(s,a) = \cup_{C \triangleright E \in \mathsf{ce}(a), C \subseteq s} E$. 
Thus, the transition function is defined as $\theta(s,a) = (s \setminus \mathsf{eff}^-(s,a))\cup \mathsf{eff}^+(s,a)$, where $\mathsf{eff}^-(s,a)$ and $\mathsf{eff}^+(s,a)$ correspond to the negative and positive effects in $\mathsf{eff}(s,a)$ to apply in the current state $s$.

A solution to a planning problem is a {\em plan} defined as a sequence of actions $\pi =\langle a_1,\ldots, a_n \rangle$ that applied in the initial state $s_0 = I$ generates a sequence of $n+1$ states $s_0,\ldots,s_n$ where the goal condition holds in the last one, i.e. $G \subseteq s_n$. For a plan to be valid, each action must be applicable in the corresponding state $\mathsf{pre}(a_i) \subseteq s_{i-1}$.




\subsection{Action schemas and grounding in \strips}
The aim of this work is to synthesize action schemas from demonstrations in $\strips$ language \cite{fikes1971strips}, which is a subset of PDDL. Thus, for simplicity of the paper we represent the tasks in PDDL with negations, disjunctive preconditions and conditional effects, so that any off-the-shelf (expressive enough) planning system can be used, i.e. Madagascar SAT-based planner \cite{rintanen2014madagascar}.

We denote a {\it PDDL task} as $P_{{\sf PDDL}} = \langle \mathcal{P,A},\Sigma,I,G \rangle$ where $\mathcal{P}$ and $\mathcal{A}$ are the set of predicates and action schemas respectively, $\Sigma$ is the set of objects, $I$ is the initial state and $G$ is the goal condition. 

Let us define $X$ as the possibly empty set of parameters that describe the parameterization of predicates and action schemas. Then, an {\it action schema} $\alpha[X]$ is defined as $\langle {\sf n}_\alpha, X, {\sf pre}_\alpha,{\sf add}_\alpha, {\sf del}_\alpha \rangle$ where ${\sf n}_\alpha$ is the name and $X$ is the set of parameters, both describe the header of the action, also known as action signature. Then, the body is described by preconditions, positive and negative effects. 

\begin{definition}[Well-defined Action Schema] 
\label{def:well-defined}
The action schema is {\it well-defined} if and only if the set of parameters used in the body is equal to the parameters of the action signature, i.e. $X = \{ Y : \forall p[Y] \in \{{\sf pre}_\alpha \cup {\sf add}_\alpha \cup {\sf del}_\alpha\}  \}$.
\end{definition}

\begin{example}[Schema]
\label{ex:visitall-schema}
 In Figure~\ref{sfig:visitall-schema} we have an example where the header is described with a name $n_{\sf move} = {\sf move}$ and a set of parameters $X = \{x,y\}$ both of type place. The body of the action schema consist of three sets of predicates, the set of preconditions ${\sf pre}_{\sf move} = \{{\sf agent\text{-}at}[x],{\sf connected}[x,y]\}$, the set of positive effects ${\sf add}_{\sf move} = \{{\sf visited}[y],{\sf agent\text{-}at}[y]\}$ and the negative effects ${\sf del}_{\sf move} = \{{\sf agent\text{-}at}[x]\}$. This action schema is well-defined, since the union of all predicate parameters in the body is $\{x,y\}$ which is equal to $X$.
\end{example}

The grounding of predicates $p[X] \in \mathcal{P}$ and action schemas $\alpha[X] \in \mathcal{A}$ is the valid assignment of objects $\omega \in \Sigma^{|X|}$ to their parameters, where  $\Sigma^{|X|}$ is the $|X|$-th Cartessian product of $\Sigma$ objects. Therefore, the set of propositional variables or fluents is $F = \{ p[\omega] : p[X] \in \mathcal{P}, \omega \in \Sigma^{|X|} \}$ which means, for every predicate there is a parameter assignment of $|X|$ objects. Thus, grounded actions are $A = \{ \alpha[\omega] : \alpha[X] \in \mathcal{A},\omega \in \Sigma^{|X|} \}$.

\begin{example}[Grounding]
\label{ex:grounding}
 When grounding Example~\ref{ex:visitall-schema}, if the set of objects is $\Sigma = \{{\sf p1, p2}\}$,
 which are two different places in visitall domain, the predicates and actions are grounded as
 follows: $F = \{ {\sf at[p1]},{\sf at[p2]},{\sf v[p1]},{\sf v[p2]}, $ ${\sf conn[p1,p1]},{\sf conn[p1,p2]},{\sf conn[p2,p1]},{\sf conn[p2,p2]}\}$, and the actions are $A = \{ {\sf m[p1,p1]},{\sf m[p1,p2]},{\sf m[p2,p1]},{\sf m[p2,p2]} \}$. We have simplified notation, such that ${\sf agent\text{-}at}$ is ${\sf at}$, ${\sf visited}$ is ${\sf v}$, ${\sf connected}$ is ${\sf conn}$, and ${\sf move}$ is ${\sf m}$. Once predicates and actions are grounded, the set of reachable states will depend on the initial state of each instance, even though there are some other grounding strategies \cite{gnad2019learning}.
\end{example}



\subsection{Learning action schemas from plan traces} 
\label{sec:learning-action-models}
The problem of learning action schemas is usually formalized as a function that minimizes a cost in the current model from observing an agent acting in the environment. In this section we describe these observations in plan traces, and formalize the inputs/outputs for the task of learning action models.

\subsubsection{Plan traces} 
There are many representation mechanisms for learning action models \cite{arora2018review}, but all these systems include at least a set of plan traces or transitions between states as part of their input.

\begin{definition}[Plan Trace] A plan trace $\tau$ is a finite sequence of states and actions, i.e. $\tau = \langle s_0,a_1,\ldots,a_n,s_n\rangle$. In that, every state $s_i$ has been reached from applying action $a_i$ in state $s_{i-1}$.
\end{definition}

\paragraph{Trace generation.} Plan traces are intuitively generated in real scenarios by observing an agent interact with an environment. This can easily be done by a software where situations and decisions taken by the agent are recorded. For instance, a domotic house where the housekeeper enter the house and shutdown the alarm, is a clear scenario where states and actions can be observed, and in which ``shutdown alarm'' action can be learned from observations. Former example shows a trace that is {\it goal-oriented}, but traces can also be represented by {\it random-walks} \cite{amir2008learning} like a robot exploring possible action outcomes in a local scenario which allows to learn action models in an {\it online} setting. However, we follow the standard {\it offline} approach in the field, where action models are learned from a given set of traces.



\paragraph{Trace observability.} There are multiple observation levels for every input plan trace $\tau$ \cite{aineto2019learning}. We denote $\Psi$  to the function that observes states and $\Phi$ to the one that observe actions from each plan trace $\tau$. In full observability environments, $\Psi : S \rightarrow S$ and $\Phi : A \rightarrow A$ always produce the original state and action respectively. In environments with no observability both functions output $\bot$, which corresponds to no observed action or state. Although, in partial observability both functions could either observe the correct state and action or $\bot$, i.e. $\Psi : S \rightarrow S \cup \{\bot\} $ and $\Phi : A \rightarrow A \cup \{\bot\} $. The initial state and goal condition (or goal state if its fully defined) are always known for any {\it goal-oriented} trace. Noise in observations is orthogonal to the observation granularity, in that, other mechanisms can be considered to handle noisy states and actions \cite{mourao2012learning,zhuo2013action} in traces, however, we assume noise-free observability.

\subsubsection{Learning $\strips$ Action Models}
\label{ssec:learning-task}
The task of learning $\strips$ action models usually requires more than one plan trace as input. Let us define a set of traces as $\mathcal{T} = \{\tau_i | 1 \leq i \leq m\}$, where $m$ is the number of training samples.


\begin{definition}[Learning Action Models Task] is a tuple $\Lambda = \langle \mathcal{P}, \mathcal{A}', \Sigma, \mathcal{T}, \mathcal{F}\rangle$ where $\mathcal{P}$ is the set of predicates, $\mathcal{A}'$ is a set of partially defined action models, $\Sigma$ is the set of objects, $\mathcal{T}$ is the set of traces, and $\mathcal{F}$ is the cost function to minimize.
\end{definition}

A solution to $\Lambda$ is a complete set of action models $\mathcal{A}$, such that $\mathcal{A}' \subseteq \mathcal{A}$. There must be a guarantee that $\mathcal{A}$ minimizes $\mathcal{F}$ in $\mathcal{T}$. The set of predicates $\mathcal{P}$ is used for learning preconditions and effects of action models, while the set of objects $\Sigma$ is used to validate that $\mathcal{A}$ can correctly generate each planning trace $\tau \in \mathcal{T}$.

The learning process of action models highly depends on the cost function specification $\mathcal{F}$. Often, this is different in every approach, ARMS counts errors over all possible predicates \cite{yang2007learning}, and FAMA uses classical metrics such as precision and recall \cite{aineto2019learning}. Then, for any given cost function metric, two models can be compared to decide which one better represents the set of plan traces. However, when comparing a model with the ground truth (if that exists), only syntactical errors should be taken into account \cite{aineto2019learning}, because semantics do not affect to the dynamics of the environment.




\begin{table}
\centering
\begin{tabular}{|l|l|l|}
\hline
 \multirow{7}{*}{Input to $\Lambda$} & \multirow{3}{*}{$\mathcal{P}$} & 
 ${\sf (connected \text{ } ?x \text{ } ?y \text{ - } place)}$,\\
 & & ${\sf (agent\text{-}at \textbf{ } ?x \text{ - } place)}$,\\
 & & ${\sf (visited \text{ } ?x \text{ - } place)}$\\
 \cline{2-3}
 & $\mathcal{A}'$ & ${\sf (:action  \textbf{  move} ( ?x \text{ } ?y \text{ - } place ) )}$ \\
 \cline{2-3}
 & $\Sigma$ & p1,p2,$\ldots$,p63,p64 - ${\sf place}$ \\
 \cline{2-3}
 & $\mathcal{T}$ & $\tau_1 = \{ s_0, {\sf move(p1\text{ } p2}), \ldots, $\\
 & & ${\sf move(p63\text{ } p64}), s_{n}\}$, $\ldots$, \\
 & & $\tau_m = \{\ldots\}$ \\
 \cline{2-3}
 &$\mathcal{F}$ & $ \sum_{\forall o \in \mathcal{O}}\frac{ {\sf predicate\text{-}errors}( o )}{{\sf all\text{-}possible\text{-}predicates}( o )}$\\
 \hline
 Output $\mathcal{A}$ & \multicolumn{2}{c|}{Action model for ${\sf move}$ in Figure~\ref{sfig:visitall-schema}}  \\
 \hline
\end{tabular}
\caption{Input/Output example for visitall problem}
\label{tab:input-output-example}
\end{table}

In Table~\ref{tab:input-output-example} we have an input/output example for the visitall in Figure~\ref{fig:visitall-example}. The input to this problem is the set of predicates introduced in Figure~\ref{sfig:visitall-schema}, the set of objects representing the 64 places in the $8 \times 8$ grid, and a partial action schema with the ${\sf move}$ action signature but whose preconditions and effects are empty. We have up to $m$ plan traces that show action sequences for visiting all grid places, and a cost function defined as the total number of errors divided by all possible predicate choices.








\section{$\strips$ Action Discovery Task}
In this section we introduce our learning task representation for classical planning, explaining the trace properties and the input requirements to guide an unsupervised search for learning valid action models.

\subsection{Task representation} 

 The representation of our problem $\Lambda^k_{r,m}$ extends that in Aineto et al. (2018) where semantics of actions (including action signatures) are unknown and the number of parameters for each of the $k$ actions is bounded by a max-arity variable $r$, and validated over $m$ examples. Action signatures highly restrict the possible preconditions and effects for learning $\strips$ action models. We relax that hard constraint in this paper to study the impact of observations for action discovery, in other words, discovering $\strips$ action headers and bodies when most of information is lost and only the initial state and objectives of each agent is known. 
 
\paragraph{Traces.} Therefore, we observe one or more agents interacting in an environment that can be defined with the same fluents and actions, also known as {\it classical planning frame} $\mathcal{D} = \langle F,A\rangle$. Agents are represented with a set of $m$ traces $\mathcal{T} = \{\tau_1,\ldots,\tau_m\}$ that are {\em noise-free} and {\em goal-oriented} in which we only observe their initial and goal states, where the latter might be partially specified. Thus, functions $\Psi$ and $\Phi$ always output $\bot$ (no observability), and traces have unknown length. So, we define each trace $\tau_i \in \mathcal{T}$, s.t. $1 \leq i \leq m$, as an initial state $s_0^i$ and a partially defined goal state $s_n^i$. For simplification, we refer to each trace as the $i$-th planning instance, i.e. $\tau_i = \langle s_0^i,s_n^i \rangle$ where $s_0^i = I_i$ and $G_i \subseteq s_n^i$. Then, $\mathcal{T} = \{P_1,\ldots,P_m\}$ where each $P_i = \langle \mathcal{D}, I_i,G_i\rangle$.
 
 \paragraph{Predicates and Objects.} Even though predicates and objects can be inferred from observed initial states and goal conditions in the traces, we provide $\mathcal{P}$ and $\Sigma$ as part of the input to ease the compilation to planning. Moreover, we extend the set objects $\Sigma^k_{r,m}$ with $k$ action, $r$ variable and $m$ trace objects, and the set $\mathcal{P}^k_r$ with predicates that represent the arities from $0$ to $r$ for each action, precondition, positive and negative effect.
 
 \paragraph{Partial Action Models.} The partially defined action models $\mathcal{A}^k_r$ consist of $k$ unnamed actions with a number of parameters that go from $0$ to $r$ that is the max-arity value. Then, the total number of partially action models is $|\mathcal{A}^k_r| = k \times (r+1)$. Each $\alpha \in \mathcal{A}^k_r$ is specified with all possible predicate combinations with {\it well-defined} parameter assignments in the preconditions (Definition~\ref{def:well-defined}), and empty effects. 
 
 
 \begin{example}[Representation]
  Following Example~\ref{ex:grounding}, the ground truth shows there is only an action $\alpha$, $k = 1$, with a max-arity $r = 2$. Then, given $\mathcal{P} = \{{\sf at[x]},{\sf v[x]},{\sf conn[x,y]}\}$, $\Sigma = \{{\sf p1,p2}\}$ and $\tau_1 = \{ s_0^1 = \{{\sf at[p1]},{\sf v[p1]},{\sf conn[p1,p2]},{\sf conn[p2,p1]}\}, s_n^1 = \{{\sf v[p1]},{\sf v[p2]}\} \}$. The partially defined action models ($n_\alpha[X]$) are $\alpha_0[]$, $\alpha_1[x]$ and $\alpha_2[x,y]$. Since all predicates have at least one parameter, $\alpha_0$ has no valid assignments in its preconditions so ${\sf pre}_{\alpha_0} = \emptyset$, ${\sf pre}_{\alpha_1} = \{{\sf at[x]},{\sf v[x]}\}$, and ${\sf pre}_{\alpha_2} = \{ {\sf at[x]},{\sf at[y]},{\sf v[x]},{\sf v[y]},{\sf conn[x,y]},{\sf conn[y,x]} \}$, while the effects of all action models are empty ${\sf add}_{\alpha_j} = \emptyset$, and ${\sf del}_{\alpha_j} = \emptyset$ (for all $0 \leq j \leq r$).
 \end{example} 
 
 The reason to remove preconditions (delete predicates) and program effects (add predicates) is to
 keep action models as restrictive as possible while having a minimum impact in the environment, so that action preconditions are relaxed and effects incremented  when more states are observed, until a convergence. This also solves the problem of {\it learning static predicates} \cite{gregory2015domain}, where deciding if they must be included or not in the action preconditions becomes a complex task, so it is much more intuitive and easy to solve when all predicate combinations are included in the preconditions and the decision is about when they should {\it not} be there.
 
 \paragraph{Cost Function.} We intuitively encode the cost function $\mathcal{F}$ as a minimization in the number of editions, where an edition is either removing or programming a predicate in the action model. This can also be reformulated as programming the minimum number of adding ($|{\sf add}_\alpha|$) and delete ($|{\sf del}_\alpha|$) effects, while removing the minimum number of predicates from preconditions ($-|{\sf pre}_\alpha|$), i.e. Equation~\ref{eq:cost-function}.
 
 
  \begin{equation}
 \label{eq:cost-function}
     \mathcal{F} = \frac{1}{|\mathcal{A'}|} \sum_{\alpha \in \mathcal{A'}} ( |{\sf add}_\alpha| + |{\sf del}_\alpha| - |{\sf pre}_\alpha| )
 \end{equation}
 
 
 The cost function in Equation~\ref{eq:cost-function} guides the search to action schemas that could be completely different from the ground truth (if that exists), but this could be still consistent and correct, since complex actions (with many parameters) can be splitted into multiple simpler actions (less parameters) \cite{areces2014optimizing}. Now we are ready to define our {\it unsupervised learning task} for action discovery as $\Lambda^k_{r,m} = \langle \mathcal{P}^k_r, \mathcal{A}^k_r, \Sigma^k_{r,m}, \mathcal{T}, \mathcal{F} \rangle$, where $k$ is the number of actions in the domain and $r$ is the bound to represent the max-arity for each action. A solution to $\Lambda^k_{r,m}$ is a set of action models $\mathcal{A}$ that subsumes $\mathcal{A}^k_r$ and solve each $\tau \in \mathcal{T}$.

\subsection{Computing $\strips$ Action Models} 
The computation of our learning task $\Lambda^k_{r,m}$ follows a {\em compilation-based} approach to PDDL, such that experiments are reproducible and any off-the-shelf classical planner can be used\footnote{In case the paper is accepted, we are going to release the framework too.}. Firstly, we explain the PDDL representation of each one of the task components.

\subsubsection{Objects, Predicates and Fluents}
The set of objects $\Sigma^k_{r,m}$ is extended with four new types corresponding to variables that bind predicate parameters to action parameters, i.e. $\Sigma_v$; copies of the original predicate names $\Sigma_p$ to relate what is being removed or added in the action model; action objects $\Sigma_a$ to indentify the action that is being updated or applied; and $\Sigma_t$ is the set of trace objects. Thus, $\Sigma^k_{r,m} = \Sigma \cup \Sigma_v \cup \Sigma_p \cup \Sigma_a \cup \Sigma_t$ where:
\begin{itemize}
\item $\Sigma_v = \{ {\sf var}_i : 1 \leq i \leq r \}$,
\item $\Sigma_p = \{ {\sf pred}_p : p \in \mathcal{P} \}$,
\item $\Sigma_a = \{ {\sf act}_i : 1 \leq i \leq k \}$,
\item $\Sigma_t = \{ {\sf tr}_i : 1 \leq i \leq m \}$.
\end{itemize}


We describe the following fluents $F'$ as the relation between predicates $\mathcal{P}^k_r$ and objects $\Sigma^k_{r,m}$. There are three kinds of fluents where the first one represents the mode $F_{\sf mode}$, if it is learning or validating the model; the second $F_{\sf ed}$ consists of removed preconditions and programmed effects from the action schemas; and the last one $F_{\sf pred}$, represents when a fluent from $F$ holds in the current state of a trace $\tau \in \mathcal{T}$. Then, $F' = F_{\sf mode} \cup F_{\sf ed} \cup F_{\sf pred}$ where :


\begin{itemize}
    \item $F_{\sf mode} = \{ {\sf edit\text{-}mode}, {\sf val\text{-}mode} \}$,
    \item $F_{\sf ed} = \{{\sf rpre}^\alpha_{p,\sigma},{\sf add}^\alpha_{p,\sigma},{\sf del}^\alpha_{p,\sigma} : \alpha \in \Sigma_\alpha, p \in \Sigma_p, \sigma \in \Sigma^{{\sf ar}(p)}_v\}$,
    \item $F_{\sf pred} = \{ {\sf holds}^\tau_{p,\omega} : \tau \in \Sigma_t, p \in \Sigma_p, \omega \in \Sigma^{{\sf ar}(p)}\}$.
\end{itemize}

To simplify fluents notation, we refer to action objects as $\alpha \in \Sigma_a$, predicate objects $p \in \Sigma_p$, trace objects $\tau \in \Sigma_t$. Also, we refer to $\omega \in \Sigma^{{\sf ar}(p)}$ as a tuple of objects assigned to a predicate $p$ with arity ${{\sf ar}(p)}$, so $\sigma \in \Sigma^{{\sf ar}(p)}_v$ is a tuple of ${\sf ar}(p)$ variable objects assigned to edition predicates that points to the corresponding action parameters.

\subsubsection{Editing and Applying Actions}

The first task of learning action models $\Lambda^k_{r,m}$ consists of editing the partially defined action models $\mathcal{A}^k_r$ by removing preconditions and programming effects:

\begin{align*}
\mathsf{pre}( \mathsf{edit}\text{-}{\sf ins}^\alpha_{p,\sigma}) =& \{ {\sf edit\text{-}mode} \} \cup \\ &\{ {\sf var}_i \neq {\sf var}_j : {\sf var}_i \in \sigma, 1 \leq i \leq {\sf ar}(\alpha) < j \},\\
\mathsf{ce}( \mathsf{edit}\text{-}{\sf ins}^\alpha_{p,\sigma}) =& \{\emptyset\} \triangleright \{ {\sf ins}^\alpha_{p,\sigma}  \}.
\end{align*}

The second precondition in edit actions is a constraint that guarantee well-defined action schemas, and word ${\sf ins}$ has to be substituted either by ${\sf rpre}$, ${\sf add}$ or ${\sf del}$ predicates. Once the edition of action models finishes, there is an action to start the validation phase:

\begin{align*}
\mathsf{pre}( \mathsf{ed2val}) =& \{ {\sf edit\text{-}mode} \},\\
\mathsf{ce}( \mathsf{ed2val}) =& \{\emptyset\} \triangleright \{ \neg {\sf edit\text{-}mode}, {\sf val\text{-}mode} \}.
\end{align*}

The computing continues applying the edited actions until the goal condition is reached for each trace:

\begin{align*}
\mathsf{pre}( \mathsf{apply}^{\alpha,\tau}_{\omega}) =& \{ {\sf val\text{-}mode} \} \cup \\ &\{ \neg {\sf rpre}^\alpha_{p,\sigma} \Rightarrow {\sf holds}^\tau_{p,\omega} \}_{\forall p \in \Sigma_p,\sigma \in \Sigma^{{\sf ar}(p)}_v},\\
\mathsf{ce}( \mathsf{apply}^{\alpha,\tau}_{\omega}) =& \{\{ {\sf add}^\alpha_{p,\sigma} \} \triangleright \{ {\sf holds}^\tau_{p,\omega} \} \cup \\ & \{{\sf del}^\alpha_{p,\sigma} \} \triangleright \{ \neg{\sf holds}^\tau_{p,\omega} \} \}_{\forall p \in \Sigma_p,\sigma \in \Sigma^{{\sf ar}(p)}_v}.
\end{align*}

\begin{example}[Variables and objects]
The relation between $\sigma$ variables and $\omega$ objects for each predicate must be well-defined, i.e. in Figure~\ref{sfig:visitall-schema}, the positive effect $({\sf visited\text{ }?y}\text{ - }{\sf place})$ is encoded in the action schema as ${\sf add^{move}_{visited,\langle var_2 \rangle}}$, so applying a move action between positions ${\sf p1}$ and ${\sf p2}$, grounds the action signature to ${\sf apply}^{{\sf move},{\sf tr}_1}_{\langle{\sf p1,p2}\rangle}$, and the positive effect will be ${\sf holds}^{{\sf tr}_1}_{{\sf visited},\langle {\sf p2} \rangle}$ which means ${\sf p2}$ has been visited in the first trace.
\end{example}

\subsubsection{DAM Compilation}
\label{sssec:dam-compilation}

We named our compilation {\bf D}iscovering {\bf A}ction {\bf M}odels (DAM). It is a compilation to a classical planning problem $P' = \langle F',A',I',G' \rangle$ where the set of fluents is $F' = F_{\sf mode}\cup F_{\sf ed} \cup F_{\sf pred}$. A set of actions for each one of the three phases: the edition phase $A_{\sf ed} = \{{\sf edit\text{-}ins}^\alpha_{p,\omega} : \alpha \in \Sigma_a, p \in \Sigma_p, \sigma \in \Sigma^{{\sf ar}(p)}_v \}$; the transition from edition to validation $A_{{\sf e2v}} = \{ {\sf ed2val}\}$; and the trace validation phase with applying actions $A_{\sf app} = \{ {\sf apply}^{\alpha,\tau}_{\omega} : \alpha \in \Sigma_a, \tau \in \Sigma_t, \omega \in \Sigma^{{\sf ar}(\alpha)} \}$. Thus, $A' = A_{\sf ed} \cup A_{{\sf e2v}} \cup A_{\sf app}$. The new initial state and goal condition are defined as the set of all fluents that hold in the initial state and the goal condition of every trace, i.e. $I' = \{ {\sf holds}^i_{p,\omega} : p[\omega] \in I_i, 1 \leq i \leq m \}$, and $G' = \{{\sf holds}^i_{p,\omega} : p[\omega] \in G_i, 1 \leq i \leq m \}$.


\subsection{Theoretical properties}
In this section we prove the theoretical properties of soundness and completeness of our approach given the number of actions $k$, the max-arity $r$ and the set of $m$ traces.

\begin{thm}[Soundness]
Any classical plan $\pi$ that solves $P'$ induces a set of action models $\mathcal{A}$ that solves $\Lambda^k_{r,m} = \langle \mathcal{P}^k_r,\mathcal{A}^k_{r},\Sigma^k_{r,m},\mathcal{T},\mathcal{F}\rangle$.
\label{thm:soundness}
\end{thm}

\begin{proof}[Proof sketch]
The first actions in $\pi$ are $A_{\sf ed}$ because of ${\sf edit\text{-}mode}$ fluent, in that, preconditions are removed and effects are programmed for each partially defined action model $\mathcal{A}^k_r$. The next is an action $A_{\sf e2v}$ which finishes editing actions and starts the validation phase. Once in validation, a sequence of edited actions from $A_{\sf app}$ is computed for each trace $\tau \in \mathcal{T}$ in parallel, mapping all initial states $I_i$ to goal conditions $G_i$ s.t. $1 \leq i \leq m$. Thus, all traces are solved by the set of edited action models $\mathcal{A}$ induced from $A_{\sf ed} \cap \pi$, which is the definition of $\mathcal{A}$ solving $\Lambda^k_{r,m}$.
\end{proof}

\begin{thm}[Completeness] Any set of action models $\mathcal{A}$ that solves $\Lambda^k_{r,m}$ can compute a plan $\pi$ that solves $P' = \langle F',A',I',G' \rangle$.
\label{thm:completeness}
\end{thm}

\begin{proof}[Proof sketch]
Action models $\mathcal{A}$  do not require any edition since they are already a solution of the learning task $\Lambda^k_{r,m}$. Moreover, a sequence of grounded $\mathcal{A}$, named $\pi_\Lambda$, solve all traces $\tau \in \mathcal{T}$ by definition. Thus, each action in $\pi_\Lambda$ can be compiled to an action in $A_{\sf app}$, which applied in the same sequence in $\pi$ map each initial state $I_i$ to its goal condition $G_i$ s.t. $1\leq i \leq m$, in other words $\pi$ maps the initial state $I'$ to a goal condition $G'$, which encode the initial states and goal conditions of all traces, solving $P'$.
\end{proof}

\subsection{Extension to Partial Observability}
\label{ssec:dam-extensions}
The inputs for learning action models usually have some level of observability, and not the extreme view we assume where only initial states and goal conditions are observed. The DAM compilation can easily be extended by adding {\em validating actions} such as \cite{aineto2019learning,aineto2018learning}, that 
specify the order in which intermediate states must hold, and a goal condition extended with a new fluent that assess when intermediate states have been correctly validated. However, the problem can be simplified to vanilla DAM, where every pair of consecutive and fully defined states in the traces can be encoded as new planning problems of initial and goal states. For instance, a single trace $\tau = \langle s_0,s_1,s_2\rangle$ can be splitted into $\tau_1 = \langle s_0, s_1 \rangle$ and $\tau_2 = \langle s_1, s_2 \rangle$ which is the input to standard DAM compilation.

\section{Unsupervised Discovering Action Models}
\label{sec:conf-space}

The DAM compilation for computing $\strips$ action models assumes the number of action schemas and max-arity bounds are given. UDAM navigates through the space of configurations and invokes repeatadly the DAM compilation. The overall process is shown in Algorithm~\ref{alg:dam}. The phases are as follows:

\begin{algorithm}
\caption{Unsupervised Discovering Action Models (UDAM)}
\label{alg:dam}
\hspace*{\algorithmicindent} \textbf{Input:} $\mathcal{P}$ predicates, $\Sigma$ objects and $\mathcal{T}$ traces \\
\hspace*{\algorithmicindent} \textbf{Output:} $\mathcal{A}$ action models
\begin{algorithmic}[1]
\STATE $\pi^* \leftarrow \emptyset$
\STATE $k, r, m \leftarrow  1, 0, |\mathcal{T}|$
\STATE $P' \leftarrow {\sf compile}( \Lambda^k_{r,m} )$ 
\STATE $open \leftarrow \{ \langle P', k,r \rangle \}$
\STATE $close \leftarrow \emptyset$
\WHILE{ $open \neq \emptyset$ }
\STATE $\langle P', k,r \rangle \leftarrow argmin_{\langle P', k,r \rangle  \in open} |{\sf operators}(P')|$
\STATE $open \leftarrow open \setminus \{\langle P', k,r \rangle \}$
\STATE $close \leftarrow close \cup \{\langle P', k,r \rangle\}$
\STATE $\pi \leftarrow {\sf madagascar}(P')$
\IF{ ${\sf sat}(\pi)$ \AND $({\sf cost}(\pi) < {\sf cost}(\pi^*) $ \OR $\pi^* = \emptyset)$ }
\STATE $\pi^* \leftarrow \pi$
\ENDIF
\STATE $k' \leftarrow k + 1$
\STATE $P'_1 \leftarrow {\sf compile}( \Lambda^{k'}_{r,m})$
\IF{ $k' \leq 2\times |\mathcal{P}|$ \AND $\langle P'_1, k', r\rangle \not\in \{open \cup close\}$}
\STATE $open \leftarrow open \cup \{ \langle P'_1, k', r\rangle \}$
\ENDIF
\STATE $r' \leftarrow r + 1$
\STATE $P'_2 \leftarrow {\sf compile}( \Lambda^{k}_{r',m})$
\IF{ $r' \leq |\Sigma|$ \AND $\langle P'_2, k, r'\rangle \not\in \{open \cup close\}$ }
\STATE $open\leftarrow open \cup \{ \langle P'_2, k, r'\rangle \}$
\ENDIF
\ENDWHILE

\STATE $\mathcal{A} \leftarrow {\sf induce}(\pi^*)$
\STATE return $\mathcal{A}$
\end{algorithmic}
\end{algorithm}

\paragraph{Phase 1 - Initializing.} Lines 1 to 5 initialize the variables, where the best plan $\pi^*$ is empty, there is one action $k = 1$ without arity $r = 0$, the number of traces is fixed $m = |\mathcal{T}|$, and a new planning problem is generated with ${\sf compile}(\Lambda^k_{r,m})$ which is defined in the DAM Compilation Section. Then, there are open and close lists, to denote visited and expanded configuration nodes, where each configuration is a parameter assignment to the learning task.

\paragraph{Phase 2 - Searching.} While there are configurations to explore in the open list, the one with lowest number of grounded operators is selected, i.e. $|{\sf operators}(P')|$. Therefore, the selected configuration is removed from the open list and added to the close list. Then, a plan $\pi$ is computed for problem $P'$ using a SAT-based planner called Madagascar~\cite{rintanen2014madagascar}. If $\pi$ is satisfiable, i.e. ${\sf sat}(\pi)$, and its cost defined in Equation~\ref{eq:cost-function} is lower than the cost of the best plan $\pi^*$ or if it is empty, then there is a new best plan. 

\paragraph{Phase 3 - Expanding.} Two new compilations are added to the open list, if they have not been previously explored. The first increments by one the number of action signatures, while the second increments by one the max-arity. Actions are theoretically bounded in the number of predicates mutiplied by 2, which means actions to add and delete each predicate. Arity is bounded in the number of objects, which means an object for each parameter of an action. It continues in Phase 2 until the open list is empty.

\paragraph{Phase 4 - Inducing.} Once all available configurations have been explored, the set of action models $\mathcal{A}$ is induced from the best plan $\pi^*$ following the proof in Theorem~\ref{thm:soundness}. 

Other strategies can be applied to Algorithm~\ref{alg:dam} such as search for first configuration that is satisfiable, or use a different function to select the next tuple from the open list. These changes will might induce different sets of action models but they will still be sound and complete, since they only change the order in which the space of configurations is explored.

\section{Experiments}
\label{sec:experiments}
The two main evaluations for learning action models consist of their computation and posterior validation on new problems. Even though, UDAM algorithm can compute $\strips$ action models from predicates, objects and traces, it requires a human interpretation if we want to explain the generated model and check for a valid model generalization. Previous methods for model learning use existing planning models from the International Planning Competitions (IPCs) \cite{lopez2015deterministic}. In that domains the dynamics are known, so they become the perfect benchmark to test learnt action similarity, or even how different number of actions and parameters relate to the original domain. The intuitive outcome from this, is that real domains can be learnt with UDAM if the learnt models are similar to expert models.

\begin{table*}[t]
\centering
\caption{Results for classical problem of the IPC. \textit{Alg. total user time} is the total CPU time taken by the algorithm. \textit{User time first sol.}, \textit{Memory first sol. (MB)} are, respectively the CPU time spent and the peak memory usage for finding the first action model. \textit{User time best sol. (s)} and \textit{Memory best sol. (MB)} are the CPU time and peak memory usage for finding the model with the least cost. \textit{Lowest cost} is the minimum cost among the found action models. \textit{Coverage best sol.} is the fraction of validation instances that have been solved with the best model found.}
\label{tab:results}
\begin{tabular}{|r|r|r|r|r|r|r|r|}
\hline
\multicolumn{1}{|c|}{\textbf{Domain}} & \multicolumn{1}{c|}{\textbf{\begin{tabular}[c]{@{}c@{}}Alg. total\\ user time (s)\end{tabular}}} & \multicolumn{1}{c|}{\textbf{\begin{tabular}[c]{@{}c@{}}User time\\ first sol. (s)\end{tabular}}} & \multicolumn{1}{c|}{\textbf{\begin{tabular}[c]{@{}c@{}}Memory first\\ sol. (MB)\end{tabular}}} & \multicolumn{1}{c|}{\textbf{\begin{tabular}[c]{@{}c@{}}User time\\ best sol. (s)\end{tabular}}} & \multicolumn{1}{c|}{\textbf{\begin{tabular}[c]{@{}c@{}}Memory best\\ sol. (MB)\end{tabular}}} & \multicolumn{1}{c|}{\textbf{\begin{tabular}[c]{@{}c@{}}Lowest\\ cost\end{tabular}}} & \multicolumn{1}{c|}{\textbf{\begin{tabular}[c]{@{}c@{}}Coverage\\ best sol.\end{tabular}}} \\ \hline
hanoi                                 & 1580.43                                                                                          & 1411.02                                                                                          & 41.71                                                                                          & 1485.90                                                                                         & 301.18                                                                                        & -1.50                                                                               & 30/30                                                                                      \\ \hline
blocks                                & 344.44                                                                                           & 114.18                                                                                           & 17.33                                                                                          & 319.95                                                                                          & 127.97                                                                                        & -3.14                                                                               & 30/30                                                                                      \\ \hline
visitall                              & 283.47.08                                                                                           & 150.69                                                                                            & 14.88                                                                                          & 199.56                                                                                           & 32.44                                                                                         & -0.80                                                                               & 30/30                                                                                      \\ \hline
\end{tabular}
\end{table*}

\paragraph{Settings.}  All experiments are evaluated using a SAT-based planner called Madagascar~\cite{rintanen2014madagascar}, a memory bound of 8GB and a time-out for each instance is set to 30 minutes in an Intel(R) Core(TM) i7-7700HQ CPU @ 2.80GHz laptop.

\paragraph{Input.}
The domains and problems from IPCs are conveniently gathered
in the PLANNING.DOMAINS repository~\cite{muise2016planning.domains}. Selected domains are: \textit{hanoi}, where a certain amount of disks must be placed from leftmost peg to rightmost one; \textit{blocks}, where towers of blocks in a given setting have to be arranged in a different order, and \textit{visitall} where an agent must visit all cells in a grid. In case of no observability, we only need the planning instances, but with partial observability, we can extract partial sequences
of states by running a planner such as Fast Downward~\cite{helmert2006fast}, that can be configured to find either a satisfying or an optimal solution. The set of predicates and objects are known from planning instances. We have fed our algorithm with 5 traces from each domain.



\subsection{Computing action models}

We measure the performance of the computation in terms of time and memory usage. These results are reported in Table~\ref{tab:results}. The reported figures are averaged over 6 executions of the algorithm, each time with 5 traces. An interesting phenomena that can be noted is that UDAM spends most of the time finding the first solution. Once the first model is found, further models are found much more quickly. The reason is that Madagascar takes much more time to ascertain that a configuration is not viable than to find a solution for a viable configuration. Most of the configurations tried at the beginning by UDAM do not have solution, so a large amount of time is spent to rule out these configurations.

The first solution found is, typically, one whose preconditions are excessively relaxed. For instance, in the case of the \textit{hanoi} domain, the first solution that is found is one with two actions with arity two. One of these actions resembles a \textit{unstack} operator from the \textit{blocksworld} domain, while the other one resembles a \textit{stack} operator. However, \textit{hanoi} does not define any predicate to indicate that a disk is grasped, so there is no way to ensure that the \textit{unstack}-like operator will be applied to a disk that is not currenlty stacked on top of another one.

Since UDAM prefers models with a larger number of preconditions in comparison to the number of effects, it will retain models in which the natural actions that an human expert would suggest are broken into several smaller actions that average a larger number of preconditions. This could represent an issue since these models allow transitions that would not be allowed by the expert's domain.

\subsection{Validation}

Precision and recall have been typical metrics to analyze machine learning models. Recently, they have been adapted to compute action models~\cite{aineto2018learning}. One of the main limitations to consider this metrics in planning, is that ground truth models are required to know how far are solutions from human expert models. In addition, when the action signatures are not known beforehand, there is no baseline to compare, increasing the difficulty of action interpretation.

In our approach we assume that ground truth model is unknown, if that exists. Thus, the best model is computed by minimizing the cost function defined in Equation \ref{eq:cost-function}. Notice that this cost can be negative if, in average, there are more preconditions than effects.

To validate, we compute the \textit{coverage} over a set of validation instances different from the one used by the algorithm to infer the models. A total of 30 instances have been used to compute the coverage. We see that, even without the action signatures, our algorithm is able to find an action model that can solve all the validation instances. While this means that the found models can solve the instances they are meant to solve, they potentially allow transitions that are not allowed by the expert's domains. We are currently seeking ways to mitigate this shortcoming.

\section{Related Work}
Learning action models has been a topic of interest for long time \cite{gil1994learning,benson1995inductive,wang1995learning}. In those, the process of generating knowledge was incremental, and planning operators were refined by observing traces of an agent interacting in complex environments. Then, ARMS algorithm \cite{yang2007learning} was capable to generate deterministic planning operators, while SLAF \cite{amir2008learning} was learning partially observable operators from predicates, objects and sequences of actions. Recent work in learning action models \cite{aineto2019learning,aineto2018learning} were capable of computing $\strips$ operators from only initial and goal states using a classical planner. We have built our compilation based on the latter, where not only preconditions and effects must be learnt but action signatures too. Also, we solve the problem of programming static predicates by initializing partial action schemas with all possible predicate-parameter combinations and applying actions that remove preconditions. In addition, we propose an unsupervised algorithm (UDAM) capable of learning action models without any bound specification, just using theoretical bounds.

In contrast to {\em system-centric} approaches \cite{arora2018review}, there is the LOCM family of algorithms \cite{cresswell2009acquisition,cresswell2011generalised,gregory2015domain}. In those, the inputs are full observable traces of actions without any state or predicate information, producing an {\em object-centric} domain represented with finite state machines. However, planning actions usually interact with multiple objects at the same time, and depending on the level of trace observation it affects to learn the correct transition system as it is highly dependant on observed object types.

There have been further research that require statistical analysis because inputs are noisy. For instance, when the actions observed in plan traces could be wrong observations \cite{zhuo2013action}, or when states in plan traces are partially observable and noisy \cite{mourao2012learning}. Also, model learning can be applied to more expressive solutions beyond primitive actions, i.e. learning hierarchical task networks \cite{zhuo2014learning}, generating planning operators {\em on-demand} with a model-free approach \cite{martinez2017relational}, or the CAMA method for learning the models from human annotators (crowdsourcing) and plan traces \cite{zhuo2015crowdsourced}. All these approaches are orthogonal to UDAM in that they can be applied to solve the partial/noise observability in traces and guiding cost functions from crowd knowledge.

Last but not least, most work in learning action models, like ARMS or FAMA, assume all action signatures are observable in the plan traces or in a partial model of action schemas. Although, there is some work that explores to learn models from non-symbolic inputs to plan in the latent space such as AMA2 \cite{asai2018classical} and LatPlan for goal recognition \cite{amado2018goal}. Both require all transitions as part of the input and only the latter generates PDDL action schemas (without guarantees) from the bits that change in the latent space. In our case, the learning task is similar to AMA2 or LatPlan, in the sense that has no information about actions, and akin to FAMA because intermediate states are non-observable. Thus, to the best of our knowledge, UDAM is the first unsupervised approach to generate $\strips$ action models from non-observable traces, where only initial states and goal conditions are known.


\section{Conclusions}

In this paper we have explored the unsupervised generation of action models, where not only preconditions and effects are learnt but action signatures too. We have followed a {\em compilation-based} approach to classical planning, so that we can use any off-the-shelf classical planner. We propose a new method in the DAM compilation, that mitigates the problem of learning static predicates, in that preconditions are removed instead of programmed (added) as in previous approaches. In the experiments, we use the IPC benchmarks to evaluate the convergence of our approach for model learning, the performance in the computation and the generalization when validating learnt action models compared to previous methods. Our validation shows that instances that are meant to be solvable are, effectively, solved by our algorithm. However, preconditions may be too weak.

As future work we consider to introduce negative examples to improve the model quality with a more informed cost function, akin to {\em Inductive Logic Programming} \cite{muggleton1994inductive}. Also, we want to address the connection between computing minimal action schemas and {\em least commitment planning} \cite{weld1994introduction}, the generation of simpler actions (less parameters) from complex actions (with  many parameters) \cite{areces2014optimizing}, and proving that the optimal set of action models (cost function based) is as complex as finding an optimal plan \cite{helmert2008good}. In addition, we want to apply UDAM algorithm to robotics, where a human can guide a robot to learn high-level skills from sensor information. Thus, our interests rely in a valid framework that can unsupervisedly compute with guarantees a high-level model from non-symbolic data such as the following methods \cite{asai2018classical,amado2018goal,bonet2019learning}.

\section{Acknowledgments}
The research leading to these results has received funding from the EU H2020 research and innovation programme under grant agreement no.731761, IMAGINE; the HuMoUR project TIN2017-90086-R (AEI/FEDER, UE); and AEI through the María de Maeztu Seal of Excellence to IRI (MDM-2016-0656).


\bibliographystyle{aaai}
\bibliography{aaai20.bib}

\begin{thebibliography}{}

\bibitem[\protect\citeauthoryear{Aineto, Celorrio, and
  Onaindia}{2019}]{aineto2019learning}
Aineto, D.; Celorrio, S.~J.; and Onaindia, E.
\newblock 2019.
\newblock Learning action models with minimal observability.
\newblock {\em Artificial Intelligence} 275:104--137.

\bibitem[\protect\citeauthoryear{Aineto, Jim{\'e}nez, and
  Onaindia}{2018}]{aineto2018learning}
Aineto, D.; Jim{\'e}nez, S.; and Onaindia, E.
\newblock 2018.
\newblock Learning strips action models with classical planning.
\newblock In {\em International Conference on Automated Planning and
  Scheduling, ICAPS-18}.

\bibitem[\protect\citeauthoryear{Amado \bgroup et al\mbox.\egroup
  }{2018}]{amado2018goal}
Amado, L.; Pereira, R.~F.; Aires, J.; Magnaguagno, M.; Granada, R.; and
  Meneguzzi, F.
\newblock 2018.
\newblock Goal recognition in latent space.
\newblock In {\em 2018 International Joint Conference on Neural Networks
  (IJCNN)},  1--8.
\newblock IEEE.

\bibitem[\protect\citeauthoryear{Amir and Chang}{2008}]{amir2008learning}
Amir, E., and Chang, A.
\newblock 2008.
\newblock Learning partially observable deterministic action models.
\newblock {\em Journal of Artificial Intelligence Research} 33:349--402.

\bibitem[\protect\citeauthoryear{Areces \bgroup et al\mbox.\egroup
  }{2014}]{areces2014optimizing}
Areces, C.; Bustos, F.; Dominguez, M.~A.; and Hoffmann, J.
\newblock 2014.
\newblock Optimizing planning domains by automatic action schema splitting.
\newblock In {\em ICAPS}.

\bibitem[\protect\citeauthoryear{Arora \bgroup et al\mbox.\egroup
  }{2018}]{arora2018review}
Arora, A.; Fiorino, H.; Pellier, D.; M{\'e}tivier, M.; and Pesty, S.
\newblock 2018.
\newblock A review of learning planning action models.
\newblock {\em The Knowledge Engineering Review} 33.

\bibitem[\protect\citeauthoryear{Asai and Fukunaga}{2018}]{asai2018classical}
Asai, M., and Fukunaga, A.
\newblock 2018.
\newblock Classical planning in deep latent space: Bridging the
  subsymbolic-symbolic boundary.
\newblock In {\em Thirty-Second AAAI Conference on Artificial Intelligence}.

\bibitem[\protect\citeauthoryear{Benson}{1995}]{benson1995inductive}
Benson, S.
\newblock 1995.
\newblock Inductive learning of reactive action models.
\newblock In {\em Machine Learning Proceedings 1995}. Elsevier.
\newblock  47--54.

\bibitem[\protect\citeauthoryear{Blythe, Deelman, and
  Gil}{2004}]{blythe2004automatically}
Blythe, J.; Deelman, E.; and Gil, Y.
\newblock 2004.
\newblock Automatically composed workflows for grid environments.
\newblock {\em IEEE Intelligent Systems} 19(4):16--23.

\bibitem[\protect\citeauthoryear{Bonet and Geffner}{2019}]{bonet2019learning}
Bonet, B., and Geffner, H.
\newblock 2019.
\newblock Learning first-order symbolic planning representations from plain
  graphs.
\newblock {\em arXiv preprint arXiv:1909.05546}.

\bibitem[\protect\citeauthoryear{Carman, Serafini, and
  Traverso}{2003}]{carman2003web}
Carman, M.; Serafini, L.; and Traverso, P.
\newblock 2003.
\newblock Web service composition as planning.
\newblock In {\em ICAPS 2003 workshop on planning for web services},
  1636--1642.

\bibitem[\protect\citeauthoryear{Cresswell and
  Gregory}{2011}]{cresswell2011generalised}
Cresswell, S., and Gregory, P.
\newblock 2011.
\newblock Generalised domain model acquisition from action traces.
\newblock In {\em ICAPS}.

\bibitem[\protect\citeauthoryear{Cresswell, McCluskey, and
  West}{2009}]{cresswell2009acquisition}
Cresswell, S.; McCluskey, T.~L.; and West, M.~M.
\newblock 2009.
\newblock Acquisition of object-centred domain models from planning examples.
\newblock In {\em ICAPS}.

\bibitem[\protect\citeauthoryear{Fikes and Nilsson}{1971}]{fikes1971strips}
Fikes, R.~E., and Nilsson, N.~J.
\newblock 1971.
\newblock Strips: A new approach to the application of theorem proving to
  problem solving.
\newblock {\em Artificial intelligence} 2(3-4):189--208.

\bibitem[\protect\citeauthoryear{Ghallab, Nau, and
  Traverso}{2004}]{ghallab2004automated}
Ghallab, M.; Nau, D.; and Traverso, P.
\newblock 2004.
\newblock {\em Automated Planning: theory and practice}.
\newblock Elsevier.

\bibitem[\protect\citeauthoryear{Gil}{1994}]{gil1994learning}
Gil, Y.
\newblock 1994.
\newblock Learning by experimentation: Incremental refinement of incomplete
  planning domains.
\newblock In {\em Machine Learning Proceedings 1994}. Elsevier.
\newblock  87--95.

\bibitem[\protect\citeauthoryear{Gnad \bgroup et al\mbox.\egroup
  }{2019}]{gnad2019learning}
Gnad, D.; Torralba, A.; Dom{\'\i}nguez, M.; Areces, C.; and Bustos, F.
\newblock 2019.
\newblock Learning how to ground a plan--partial grounding in classical
  planning.
\newblock In {\em Proceedings of the AAAI Conference on Artificial
  Intelligence}, volume~33,  7602--7609.

\bibitem[\protect\citeauthoryear{Gregory and
  Cresswell}{2015}]{gregory2015domain}
Gregory, P., and Cresswell, S.
\newblock 2015.
\newblock Domain model acquisition in the presence of static relations in the
  lop system.
\newblock In {\em ICAPS},  97--105.

\bibitem[\protect\citeauthoryear{Helmert, R{\"o}ger, and
  others}{2008}]{helmert2008good}
Helmert, M.; R{\"o}ger, G.; et~al.
\newblock 2008.
\newblock How good is almost perfect?.
\newblock In {\em AAAI}, volume~8,  944--949.

\bibitem[\protect\citeauthoryear{Helmert}{2006}]{helmert2006fast}
Helmert, M.
\newblock 2006.
\newblock The fast downward planning system.
\newblock {\em JAIR} 26:191--246.

\bibitem[\protect\citeauthoryear{Kambhampati}{2007}]{kambhampati2007model}
Kambhampati, S.
\newblock 2007.
\newblock Model-lite planning for the web age masses: The challenges of
  planning with incomplete and evolving domain models.
\newblock In {\em Proceedings of the National Conference on Artificial
  Intelligence}, volume~22,  1601.
\newblock Menlo Park, CA; Cambridge, MA; London; AAAI Press; MIT Press; 1999.

\bibitem[\protect\citeauthoryear{L{\'o}pez, Celorrio, and
  Olaya}{2015}]{lopez2015deterministic}
L{\'o}pez, C.~L.; Celorrio, S.~J.; and Olaya, {\'A}.~G.
\newblock 2015.
\newblock The deterministic part of the seventh international planning
  competition.
\newblock {\em Artificial Intelligence} 223:82--119.

\bibitem[\protect\citeauthoryear{Mart{\'\i}nez, Alenya, and
  Torras}{2017}]{martinez2017relational}
Mart{\'\i}nez, D.; Alenya, G.; and Torras, C.
\newblock 2017.
\newblock Relational reinforcement learning with guided demonstrations.
\newblock {\em Artificial Intelligence} 247:295--312.

\bibitem[\protect\citeauthoryear{Mourao \bgroup et al\mbox.\egroup
  }{2012}]{mourao2012learning}
Mourao, K.; Zettlemoyer, L.~S.; Petrick, R.; and Steedman, M.
\newblock 2012.
\newblock Learning strips operators from noisy and incomplete observations.
\newblock {\em arXiv preprint arXiv:1210.4889}.

\bibitem[\protect\citeauthoryear{Muggleton and
  De~Raedt}{1994}]{muggleton1994inductive}
Muggleton, S., and De~Raedt, L.
\newblock 1994.
\newblock Inductive logic programming: Theory and methods.
\newblock {\em The Journal of Logic Programming} 19:629--679.

\bibitem[\protect\citeauthoryear{Muise}{2016}]{muise2016planning.domains}
Muise, C.
\newblock 2016.
\newblock {Planning.Domains}.
\newblock In {\em ICAPS - Demonstrations}.

\bibitem[\protect\citeauthoryear{Rintanen}{2014}]{rintanen2014madagascar}
Rintanen, J.
\newblock 2014.
\newblock Madagascar: Scalable planning with sat.
\newblock {\em Proceedings of the 8th International Planning Competition
  (IPC-2014)} 21.

\bibitem[\protect\citeauthoryear{Segovia-Aguas, Jim{\'e}nez, and
  Jonsson}{2017}]{segovia2017unsupervised}
Segovia-Aguas, J.; Jim{\'e}nez, S.; and Jonsson, A.
\newblock 2017.
\newblock Unsupervised classification of planning instances.
\newblock In {\em Proceedings of the Twenty-Seventh International Conference on
  Automated Planning and Scheduling (ICAPS 2017); 2017 June 18-23; Pittsburgh,
  USA. Palo Alto, CA: AAAI; 2017. p. 452-60.}
\newblock Association for the Advancement of Artificial Intelligence (AAAI).

\bibitem[\protect\citeauthoryear{Wang}{1995}]{wang1995learning}
Wang, X.
\newblock 1995.
\newblock Learning by observation and practice: An incremental approach for
  planning operator acquisition.
\newblock In {\em Machine Learning Proceedings 1995}. Elsevier.
\newblock  549--557.

\bibitem[\protect\citeauthoryear{Weld}{1994}]{weld1994introduction}
Weld, D.~S.
\newblock 1994.
\newblock An introduction to least commitment planning.
\newblock {\em AI magazine} 15(4):27.

\bibitem[\protect\citeauthoryear{Yang, Wu, and Jiang}{2007}]{yang2007learning}
Yang, Q.; Wu, K.; and Jiang, Y.
\newblock 2007.
\newblock Learning action models from plan examples using weighted max-sat.
\newblock {\em Artificial Intelligence} 171(2-3):107--143.

\bibitem[\protect\citeauthoryear{Zhuo and Kambhampati}{2013}]{zhuo2013action}
Zhuo, H.~H., and Kambhampati, S.
\newblock 2013.
\newblock Action-model acquisition from noisy plan traces.
\newblock In {\em IJCAI},  2444--2450.

\bibitem[\protect\citeauthoryear{Zhuo, Mu{\~n}oz-Avila, and
  Yang}{2014}]{zhuo2014learning}
Zhuo, H.~H.; Mu{\~n}oz-Avila, H.; and Yang, Q.
\newblock 2014.
\newblock Learning hierarchical task network domains from partially observed
  plan traces.
\newblock {\em Artificial intelligence} 212:134--157.

\bibitem[\protect\citeauthoryear{Zhuo}{2015}]{zhuo2015crowdsourced}
Zhuo, H.~H.
\newblock 2015.
\newblock Crowdsourced action-model acquisition for planning.
\newblock In {\em Twenty-Ninth AAAI Conference on Artificial Intelligence}.

\end{thebibliography}

\end{document}